\documentclass{article}

\usepackage[final]{neurips_2024}

\usepackage[utf8]{inputenc} 
\usepackage[T1]{fontenc}    
\usepackage{hyperref}       
\usepackage{url}            
\usepackage{booktabs}       
\usepackage{amsfonts}       
\usepackage{nicefrac}       
\usepackage{microtype}      
\usepackage{xcolor}   
\RequirePackage{natbib}
\RequirePackage{eso-pic} 
\RequirePackage{forloop}
\RequirePackage{url}
\usepackage{amsmath}
\usepackage{amssymb}
\usepackage{todonotes}
\usepackage{wrapfig}

\usepackage{titletoc}
\usepackage[page, header, toc, page]{appendix} 

\definecolor{darkgreen}{rgb}{0,0.3,0}
\definecolor{darkblue}{rgb}{0.0,0.0,0.4}
\definecolor{darkred}{rgb}{0.3,0.0,0.0}
\hypersetup{
	colorlinks = true,
	citecolor  = darkblue,
	linkcolor  = darkred,
	filecolor  = darkblue,
	urlcolor   = darkblue,
}

\RequirePackage{algorithm}
\RequirePackage{algorithmic}
\usepackage{url}
\usepackage{graphicx}
\usepackage{placeins}
\usepackage{caption,subcaption}
\usepackage{float}
\usepackage{dcolumn}
\usepackage{nicefrac}

\usepackage{enumerate,enumitem}
\usepackage{mathtools}
\mathtoolsset{showonlyrefs,showmanualtags}
\usepackage{mdframed}
\mdfsetup{
linecolor=black!40,
linewidth=0.5pt,
innerleftmargin = 2pt,
innertopmargin = 0pt,
innerrightmargin= 2pt,
innerbottommargin = 4pt,
}
\setlist{nolistsep}
\setlist{leftmargin=*}

\usepackage{amsmath,amscd,amssymb,amsthm} 
\usepackage{verbatim}
\usepackage{thmtools}
\usepackage{thm-restate}

\declaretheorem[name=Theorem]{theorem} 

\declaretheorem[name=Lemma, numberlike=theorem]{lemma}
\declaretheorem[sibling=theorem]{definition}

\declaretheorem[sibling=theorem]{corollary}

\declaretheorem[sibling=theorem]{assumption}
\declaretheorem[shaded={rulecolor=black, rulewidth=0.5pt, bgcolor=gray!7}, name=Theorem, sibling=theorem]{thmbox}

\declaretheorem[shaded={rulecolor=black, rulewidth=0.5pt, bgcolor=white}, name=Lemma, sibling=theorem]{lemboxlight}

\usepackage{enumerate}
\usepackage{mathtools}
\mathtoolsset{showonlyrefs} 
\usepackage{epsfig}
\usepackage{multirow}
\usepackage{graphicx}
\usepackage{epic} 
\usepackage{eepic}
\usepackage{ifthen} 
\usepackage{amsfonts}
\usepackage{wasysym}
\usepackage{color}
\usepackage{setspace}
\usepackage{booktabs} 
\usepackage{diagbox}
\usepackage{thm-restate}

\usepackage{tikz}
\usetikzlibrary{shapes}

\usepackage{url}
\usepackage{hyperref}

\newcommand{\algcomment}[1]{\textcolor{blue!70!black}{\small{\texttt{\textbf{//\hspace{2pt}#1}}}}}

\newcommand{\alg}{\mathcal{A}} 
\renewcommand{\P}{\mathbb{P}}

\def\bz{{\mathbf z}}
\def\bx{{\mathbf x}}

\def\bv{{\mathbf v}}
 
\newcommand{\D}{\mathrm{d}}

\newcommand{\bG}{\boldsymbol{G}}
\newcommand{\bsigma}{\boldsymbol{\sigma}}

\newcommand{\E}{\mathop{\mathbb{E}}}
\newcommand{\R}{\mathbb{R}}

\newcommand{\N}{\mathbb{N}}
\newcommand{\B}{\mathbb{B}}
\newcommand{\argmin}{\mathop{\text{argmin}}}

\newcommand{\by}{\mathbf{y}}
\newcommand{\bg}{\mathbf{g}}
\newcommand{\bu}{\mathbf{u}}

\newcommand{\eps}{\varepsilon}
\newcommand{\scale}{\alpha}

\newcommand{\dregret}[1]{\textup{Regret}^{[\beta]}_{#1}}

\def\barx{\overline{\mathbf x}} 

\newcommand{\norm}[1]{\left\| #1\right\|}
\newcommand{\regnorm}[2]{\left\| #2\right\|^{[#1]}}

\newcommand{\clip}{\mathrm{clip}}
\newcommand{\stograd}{\textsc{StoGrad}}

\newcommand{\inp}[2]{\left\langle #1,#2\right\rangle}

\newcommand{\abs}[1]{\left|{#1}\right|}

\newcommand*\circled[1]{\tikz[baseline=(char.base)]{\node[shape=circle,draw,inner sep=2pt] (char) {\small #1};}}

\definecolor{MITBrown}{RGB}{164, 31, 50}

\title{Adam with model exponential moving average is effective for nonconvex optimization}

\author{Kwangjun Ahn \\
 Microsoft Research\\
 Cambridge, MA
\\ \texttt{kwangjunahn@microsoft.com} 
\And Ashok Cutkosky\\
Boston University\\
 Boston, MA
\\ \texttt{ashok@cutkosky.com}
}

\begin{document}

\maketitle

\begin{abstract}
In this work, we offer a theoretical analysis of two modern optimization techniques for training large and complex models: (i) adaptive optimization algorithms, such as Adam, and (ii) the model exponential moving average (EMA). Specifically, we demonstrate that a clipped version of Adam with model EMA achieves the optimal convergence rates in various nonconvex optimization settings, both smooth and nonsmooth. Moreover, when the scale varies significantly across different coordinates, we demonstrate that the coordinate-wise adaptivity of Adam is provably advantageous. Notably, unlike previous analyses of Adam, our analysis crucially relies on its core elements---momentum and discounting factors---as well as model EMA, motivating their wide applications in practice.
\end{abstract}

\section{Introduction}
 \label{sec:intro}

In neural network training, the training loss $F:\R^d\to\R$ is often optimized using an iterative optimization algorithm which starts with the initial iterate $\bx_0$ and then updates during each iteration $t=1,2,\dots$ as follows:
\begin{align} \label{exp:update}
\bx_{t} = \bx_{t-1} + \bz_t \,,
\end{align}
where $\bz_t$ denotes the increment chosen by the algorithm during the $t$-th iteration.
One of the most popular optimization algorithms is {\bf Adam} \citep{kingma2015adam}. Adam has gained significant attention due to its effectiveness in training Transformer-based language models \citep{zhang2020why,kunstner2023noise,jiang2022does,pan2023toward,ahn2024linear,kunstner2024heavy,zhang2024transformers}.

The {\bf model exponential moving average} (EMA)~\citep{polyak1992acceleration,ruppert1988efficient} is an optimization technique that has gained popularity in conjunction with Adam for various recent applications. 
EMA maintains an exponential moving average of the model iterates, $\bx_t$, which contributes to the stabilization of these iterates. 
There has been a resurgence of interest in this technique due to its effectiveness in training high-quality generative models~\citep{yaz2018unusual,karras2019style,song2021scorebased,dhariwal2021diffusion,nichol2021improved,song2021denoising,balaji2022ediff,karras2022elucidating,rombach2022high,kang2023scaling,karras2023analyzing}. 
Moreover, a recent work by \cite{block2024butterfly} demonstrates the effectiveness of EMA for both language modeling and imitation learning applications.

In this work, we theoretically study the effectiveness of these two modern optimization techniques.
Our main results can be informally summarized as follows. 

\begin{thmbox}[Informal] \label{thm:informal}
\ref{adam} with the EMA on its iterates achieves the optimal convergence rate for nonconvex optimization both for smooth and nonsmooth settings (\autoref{sec:global}).
The coordinate-wise adaptivity of Adam is particularly effective when the scale varies across different coordinates (\autoref{sec:coordinate}).
\end{thmbox}
Our main results are based on the online-to-nonconvex conversion framework of \cite{cutkosky2023optimal}, which chooses the increment $\bz_t$ based on an online learner of choice. 
In particular, our approach is quite different than the previous analyses of Adam (see \autoref{sec:related} below). 
Notably, our analysis relies on the key components of Adam (momentum and adaptive learning rate) as well as EMA of the iterates, offering new, theoretical insight into their success.
See \autoref{sec:discussion} for a more detailed discussion.

At a high level, our analysis combines the main insights from the two recent works: \cite{zhang2024random} and \cite{ahn2024understanding}.
We first carefully modify the discounted-to-nonconvex conversion framework (\autoref{lem:o2nc}) of \cite{zhang2024random} which converts an online learner that achieves a good discounted regret (\autoref{def:dregret}) into a good noncovex optimizer.
We then combine it with the main insight of \cite{ahn2024understanding} that an effective discounted online learner can be designed based on scale-free Follow-the-Regularized-Leader (FTRL)~\citep{orabona2018scale}.
In particular, the way we arrive at Adam is similar to \cite{ahn2024understanding}: choosing a discounted version of FTRL in the discounted-to-nonconvex conversion leads to Adam.

\subsection{Related work}
\label{sec:related}

Even though Adam is widely used in deep learning, our theoretical understanding of its inner workings, especially the importance of its core components---momentum and discounting factors---remains incomplete, as pointed out by \cite{ahn2024understanding}.  
Most theoretical work on Adam and its variations focus on characterizing the convergence rate for convex or smooth nonconvex functions, where methods like SGD already achieve the minimax optimal convergence rate.~\citep{reddi2018on,zhou2018adashift,chen2019convergence,zou2019sufficient,alacaoglu2020new,guo2021novel,defossez2022simple,zhang2022adam,li2023convergence,wang2023closing}. 
In fact, even the most recent results \citep{li2023convergence,wang2023closing}  are not reflective of practice, in the sense that Adam's convergence rate worsens with momentum~\citep[\S 6]{wang2023closing} or is no better than that of SGD~\citep[\S 7]{li2023convergence}.
A notable exception is
\citet{crawshaw2022robustness}, which demonstrates the advantages of momentum under the generalized smoothness conditions of \citet{zhang2020why}.
However, the algorithm they analyze is signSGD, which differs significantly from the original Adam.
In contrast, as we will explore in the subsequent sections, momentum and discounting factors are important in our analysis.
See \autoref{sec:discussion} for more details.

We also highlight that our analysis  relies on model EMA, a technique widely used in practice as mentioned above (also see a recent work by \cite{block2024butterfly}). 
It is worth noting that EMA (or model averaging in general) has shown to have generalization benefits in practice~\citep{tarvainen2017mean,izmailov2018averaging}. In this paper, we study EMA from an optimization perspective, and show that the use of EMA leads to optimal guarantees for nonconvex optimization. 
Interestingly, EMA naturally derives from the discounted-to-online conversion (see \autoref{alg:o2nc}), which, we believe, provides new theoretical insights into this practical method.

The use of EMA also represents a significant departure from most non-convex optimization analyses. While EMA is a classical technique in the \emph{convex} setting, theoretical analyses in the non-convex setting typically randomly select an iterate as the ``final output'' of the optimizer, rather than using EMA. This random selection is intuitively extremely impractical (indeed, on average it actually wastes half of the computation), and is never performed in real implementations.

Our analysis follows a line  of work studying convergence guarantees for non-smooth non-convex optimization. Our particular convergence criterion is similar to finding the Goldstein stationary points \citep{goldstein1977optimization} that were first studied in the context of modern machine learning by \citep{zhang2020complexity}, and has seen much subsequent interest \citep{tian2022no, jordan2023deterministic, davis2020stochastic}. Other notions of convergence are also reasonable---common alternatives involve the Moreau envelope, or imposing a weak convexity condition  \citep{davis2018subgradient, davis2022escaping}.

\section{Setting for nonconvex and nonsmooth optimization}

Throughout this paper, unless specified otherwise, $\norm{\cdot}$ denotes the $L_2$ norm. 
Following \citep{cutkosky2023optimal}, we consider optimizing a loss function $F$ that satisfies the following conditions, accessing information about $F$ through a \emph{stochastic gradient oracle} $\stograd: \R^d \times \mathcal{Z} \to \R^d$, for the set of randomness $\mathcal{Z}$.

\begin{assumption} \label{assump}
Let $F:\R^d \to \R$ be a differentiable function with the following properties:
\begin{itemize}

\item Let $\Delta \coloneqq F(\bx_0) -\inf_{\bx} F(\bx)$.

\item For any two points $\bx$ and $\by$, $F(\by)-F(\bx) = \int_0^1 \inp{\nabla F(\bx+t(\by-\bx))}{\by-\bx}\D t$. 

\item {\bf Lipschitzness.} $F$ is $G$-Lipshitz, \emph{i.e.}, for any point $\bx$, $\norm{\nabla F(\bx)}\leq G$.

\item {\bf Stochastic gradient variance.} For any point $\bx$, the stochastic gradient $\bg\gets \stograd(\bx,r)$ for randomness $r\in \mathcal{Z}$ satisfies 
$\E[\bg] = \nabla F(\bx)$ and $\E\norm{\bg - \nabla F(\bx)}^2 \leq \sigma^2$.

\end{itemize}

\end{assumption}

The Lipschitz continuity condition is a standard assumption in nonconvex nonsmooth settings. However, as we will discuss in \autoref{sec:coordinate}, one of the key insights from our results is that Adam enables us to {\bf adapt} to the Lipschitz constants {\bf coordinate-wise} without requiring prior knowledge of these constants. Note that we almost certainly need some form of structural assumption on the ``difficulty'' of the loss function; thus, relaxing the Lipschitz assumption would likely come at the cost of another assumption, such as smoothness.

The second condition, called \emph{well-behavedess} in \citep[Definition 1]{cutkosky2023optimal}, is a mild regularity condition. 
For any locally Lipschitz function $F$, applying an arbitrarily small perturbation to the function is sufficient to ensure this condition \citep[Proposition 2]{cutkosky2023optimal}.

For the notion of optimality, we follow \cite{zhang2024random} and consider the following notion of stationarity for  nonconvex and nonsmooth functions.
This notion is a slight relaxation of the notion of a Goldstein stationarity point~\citep{goldstein1977optimization}, which was further studied by recent works \citep{zhang2020why,davis2022gradient,tian2022finite,jordan2023deterministic}.

\begin{definition}[{\bf $(\lambda,\eps)$-stationary point}]
\label{def:reg_station}
Suppose $F:\R^d\to \R$ is differentiable.
We say $\bx$ is a $(\lambda,\eps)$-stationary point of $F$ if $\regnorm{\lambda}{\nabla F(\bx)}  \leq \eps$, where 
\begin{align}
\regnorm{\lambda}{\nabla F(\bx)} \coloneqq \inf_{\substack{p \in \mathcal{P}(\R^d),\\\E_{\by\sim p}[\by] =\bx} } \left\{\norm{\E[\nabla F(\by)]} + \lambda \cdot \E\norm{\by-\bx}^2 \right\}\,.
\end{align}
\end{definition}

To further motivate this definition, we remark that $(\lambda,\epsilon)$-stationary points retain the desirable properties of Goldstein stationary points. Specifically, the following result {\citep[Lemma 2.3]{zhang2024random}} demonstrates that, akin to Goldstein stationary points, $(\lambda,\epsilon)$-stationary points can be reduced to first-order stationary points with appropriate choices of $\lambda$ when the objective function is smooth or second-order smooth.

\begin{lemma}  \label{lem:convert}
 If $F$ is $L$-smooth, then an $(L^2 \eps^{-1},\eps)$-stationary point $\bx$ of $F$ satisfies $\norm{\nabla F(\bx)}\leq 2\eps$. Moreover, if $F$ is $H$-second-order-smooth, then an $(H/2,\eps)$-stationary point $\bx$ of $F$ satisfies $\norm{\nabla F(\bx)}\leq 2\eps$. 
\end{lemma}

Moreover, as shown by  {\citep[Lemma 2.4]{zhang2024random}},  $(\lambda, \eps)$-stationary points can also be reduced to  
Goldstein stationary points when $F$ is Lipschitz.

\begin{lemma}  \label{lem:convert2} 
Suppose $F$ is $G$-Lipschitz. For any $\lambda,\eps,\delta>0$,
a $(\lambda,\eps)$-stationary point is a $(\delta, \eps')$-Goldstein stationary
point, where $\eps' = (1+\frac{2G}{\lambda \delta^2})\cdot \eps$.
\end{lemma}

Now we design algorithms that find $(\lambda,\eps)$-stationary points efficiently.

\section{Discounted-to-nonconvex conversion:  online learning of increments}

Our main results are built on the online-to-nonconvex conversion framework of \cite{cutkosky2023optimal}.
At its core, this framework involves selecting the increment $\bz_t$ using an online learner, as discussed by \cite{ahn2024understanding}. Specifically, we follow a variant developed by \cite{zhang2024random}, which carefully incorporates the discounting factor in the conversion process.
Note that we make slight modifications to the version proposed by \cite{zhang2024random} as follows.
Here $\text{Exp}(1)$ denotes the exponential random variable with mean $1$.

\begin{algorithm}
\caption{{\bf Discounted-to-nonconvex conversion} (choosing increments via online learning)}
\label{alg:o2nc}
\begin{algorithmic}
\STATE{\bfseries Input: } Initial point $\bx_0$, $T \in\N$, online learning algorithm $\alg$, and discounting factor $\beta\in(0,1)$
\FOR{$t=1,2\dots, T$}
\STATE Receive  $\bz_t$ from $\alg{}$  \algcomment{choose the increment using an online learner}
\STATE Update $\bx_t \gets \bx_{t-1} + \scale_t\bz_t$, where $\scale_t  {\sim} \text{Exp}(1)$  \emph{i.i.d.}
\STATE Compute $\bg_t\gets \stograd(\bx_t,r_t)$  with freshly sampled randomness $r_t$
\STATE Send $\ell^{[\beta]}_{t}(\bz) \coloneqq \inp{\beta^{-t}\bg_t}{\bz}$ to $\alg$ 
\STATE \algcomment{Maintain  exponential moving average (for output only):}
\STATE  Update $\barx_t \gets  \frac{\beta- \beta^t}{1-\beta^t}\barx_{t-1} + \frac{1-\beta}{1-\beta^t}\bx_{t}$ \quad (Equivalently, $\barx_t \gets \frac{1-\beta}{1-\beta^t} \sum_{s=1}^t\beta^{t-s}\bx_s$)
\ENDFOR
\end{algorithmic}
\end{algorithm}

Given \autoref{alg:o2nc}, it turns out we need to design an online learner that minimizes the discounted regret, formally defined below.
It is worth noting that discounted regret has been recently studied with the goal of better adapting online learners to dynamic environments~\citep{ahn2024understanding,zhang2024discounted,jacobsen2024online}.

\begin{definition}[{\bf Discounted regret}] \label{def:dregret}
For a comparator $\bu$, the $\beta$-discounted regret is defined as
\begin{align}
\dregret{t}(\bu) \coloneqq \beta^t \cdot \sum_{s=1}^t (\ell^{[\beta]}_{ s}(\bz_s)- \ell^{[\beta]}_{ s}(\bu)) = \sum_{s=1}^t \beta^{t-s}\inp{ \bg_s}{\bz_s-\bu}\,.
\end{align}
\end{definition}

The discounted regret of an online learner $\alg$ can be used to upper bound the norm of averaged gradients, as shown in the following result.

\begin{lemboxlight}[{\bf Discounted-to-nonconvex conversion}] \label{lem:o2nc}
Suppose that $F$ satisfies \autoref{assump}. 
Then for the comparator sequence chosen as  $\bu_t \coloneqq -D \frac{\sum_{s=1}^t\beta^{-s}\nabla F(\bx_s)}{\norm{\sum_{s=1}^t\beta^{-s}\nabla F(\bx_s)}}$,  \autoref{alg:o2nc} gives
\begin{align}
\E_{t\sim [T]} \E\norm{\E_{\by_t}\nabla F(\by_t)}&\leq \frac{\Delta}{DT} + \frac{2G+\sigma}{(1-\beta) T} +\sigma \sqrt{1-\beta}   \\
&\quad +\frac{1}{DT} \left[ \beta\cdot  \E \left[\dregret{T}(\bu_T)\right]  + (1-\beta)\cdot  \sum_{t=1}^T   \E \left[\dregret{t}(\bu_t)\right] \right]\,,
\end{align}
where  $\by_t$ is distributed over $\{\bx_s\}_{s=1}^t$ as $\P(\by_t = \bx_s) = \beta^{t-s} \cdot \frac{1-\beta}{1-\beta^t}$ for $s=1,2,\dots, t$.
\end{lemboxlight} 

The proof combines the techniques of \citep[Theorem 7]{cutkosky2023optimal} and \citep[Theorem 3.3]{zhang2024random}.
See \autoref{pf:lem:o2nc} for details.

We briefly explain how \autoref{lem:o2nc} can be used to find a $(\lambda,\epsilon)$-stationary point (\autoref{def:reg_station}).
Recall that $(\lambda,\epsilon)$-stationarity essentially requires producing a point $\mathbf{x} = \mathbb{E}[\mathbf{y}]$ such that both $\norm{\mathbb{E}[\nabla F(\mathbf{y})]}$ and $\mathbb{E}\|\mathbf{y}-\mathbf{x}\|^2$ are small.

Given this context, \autoref{lem:o2nc} states that as long as the discounted regret of the online learner $\mathcal{A}$ is small, we can ensure that the EMA iterates $\mathbf{\bar{x}}_t = \mathbb{E}[\mathbf{y}_t]$ serve as good candidates for $(\lambda,\epsilon)$-stationarity, since the term $\E\norm{\E_{\by_t}\nabla F(\by_t)}$ can be kept small. 
The remaining task is to bound the variance term, $\E  \norm{\by_t - \barx_t}^2$, which will be addressed later in \autoref{lem:variance}.

Moreover, the comparator $\bu_t$ roughly models the  \emph{update direction that an oracle algorithm with perfect knowledge of the loss would select}. In the proof of \autoref{lem:o2nc}, we demonstrate that moving along the $\bu_t$ direction effectively decreases the loss value, which forms the basis for establishing our convergence guarantee.

Thanks to the discounted-to-nonconvex conversion, the task now reduces to designing an online learner that achieves low discounted regret.

\section{Scale-free Follow-the-Regularized-Leader (FTRL)}

In this section, we introduce an algorithmic component, called the Followed-The-Regularized-Leader (FTRL), a powerful online learning technique with various applications \citep{gordon1999regret,kalai2005efficient,shalev2006online,abernethy2008competing,nesterov2009primal,hazan2008extracting}.

For the setting, consider the online linear optimization (OLO) setting, where during each round $t=1,\dots, T$, and online learner chooses $\bz_t$, and then the linear loss   $\ell_t(\cdot) = \inp{\bv_t}{\cdot}$ is revealed by the environment.
Here the goal of the online learner is to minimize the regret defined as $\sum_{t} \inp{\bv_t}{\bz_t -\bu}$, where $\bu$ is the comparator in hindsight.
For this setting,  FTRL is presented in \autoref{alg:ftrl}.

\begin{algorithm}[H]
   \caption{Follow-the-Regularized-Leader (FTRL)} 
\label{alg:ftrl}
\begin{algorithmic}[1]
{
\REQUIRE{Regularizers $\{\psi_t\}  :\R^d \to \R$, the domain $\mathcal{D}\subseteq \R^d$}
\FOR{$t=1,2,\dots,T$}
\STATE{Update $\bz_t \gets \argmin_{\bz \in \mathcal{D}} \ \left[\psi_t(\bz) + \sum_{s=1}^{t-1} \ell_s(\bz)\right]$}
\STATE{Receive the next loss $\ell_t(\cdot) = \inp{\bv_t}{\cdot}$ }
\ENDFOR
}
\end{algorithmic} 
\end{algorithm}

The key insight of \cite{ahn2024understanding} and \cite{zhang2024discounted} is that in order to design an online learner for discounted regret, it is important that the online learner is \emph{scale-free} as described below.
In particular, following \cite{ahn2024understanding}, we consider a gradient adaptive scale-free FTRL algorithm called \emph{scale-free FTRL} \citep{orabona2018scale}.

We will focus on the case where $\mathcal{D} = \B_D$, the $d$-dimensional $L_2$-ball of radius $D>0$. 
Scale-free FTRL is given by \autoref{alg:ftrl}  with the following choie:
\begin{align}
\psi_t(\cdot) = \frac{1}{\eta_t}\norm{\cdot}^2\quad \text{and}\quad \eta_t  = \frac{D}{\sqrt{\sum_{s=1}^{t-1} \norm{\bv_s}^2}}.    
\end{align} 
Then using the clipping operator $\clip_D(\bx):=\bx \min(\nicefrac{D}{\norm{\bx}},1)$, we can write down the update rule more explicitly as follows:
\begin{align}\tag{\textsc{scale-free FTRL}}\label{ftrl}
\bz_t \gets  \argmin_{\bz \in \B_D}\left[  
\frac{1}{\eta_t}\norm{\bz}^2+ \sum_{s=1}^{t-1} \inp{\bv_s}{\bz}  \right]  = -\clip_D\left(D \frac{\sum_{s=1}^{t-1} \bv_{s}}{\sqrt{\sum_{s=1}^{t-1} \norm{\bv_s}^2}}\right)\,.   
\end{align}
Here, if the denominator is zero, \emph{i.e.}, $\bv_1=\cdots = \bv_{t-1}=\mathbf{0}$, then we set $\bz_t\gets 0$.
Note that this algorithm is scale-free in the sense that when the loss sequence is scaled by a scalar $c>0$, the updates remain the same.

Let us now present the regret bound of \ref{ftrl}.

\begin{lemma}[Gradient-adaptive regret bound] \label{lem:ftrl}
For any $T>0$, loss sequence $\bv_{1:T}$ and comparator $\bu\in\R^d$ s.t. $\norm{\bu}\leq D$, \ref{ftrl} guarantees the following regret bound:
\begin{align}
\sum_{t=1}^{T} \inp{\bv_{t}}{\bz_t- \bu} \leq 4D \sqrt{\sum_{t=1}^{T}\norm{\bv_t}^2} \,.
\end{align}
\end{lemma}

We note that \autoref{lem:ftrl} follows (with a slightly worse constant) from  \citep[Theorem 1]{orabona2018scale}, and the version we invoke is here due to  \cite[Theorem A.1]{ahn2024understanding}.

Recall from \autoref{lem:o2nc} that an online learner for the discounted-to-nonconvex conversion (\autoref{alg:o2nc}) needs to have a low discounted regret.
To achieve this, following  \cite{ahn2024understanding} and \cite{zhang2024discounted}, we simply substitute $\bv_t \gets \beta^{-t}\bg_t$ into \ref{ftrl}, resulting in the update
\begin{align}\tag{\textsc{$\beta$-FTRL}}\label{dftrl}
\bz_t \gets   -\clip_D\left(D \frac{\sum_{s=1}^{t-1} \beta^{-s}\bg_{s}}{\sqrt{\sum_{s=1}^{t-1} \beta^{-2s}\norm{\bg_s}^2}}\right)\,.   
\end{align}
Here again, if the denominator is zero, \emph{i.e.}, $\bg_1=\cdots = \bg_{t-1}=\mathbf{0}$, then we set $\bz_t\gets 0$.
Then, the following result characterizes the discounted regret guarantee of \ref{dftrl}.

\begin{theorem}[{\bf Discounted regret bound}]\label{thm:dftrl}
Let $\beta \in (0,1]$.
For any $T>0$, loss sequence $\bg_{1:T}$ and comparator $\bu\in\R^d$ s.t. $\norm{\bu}\leq D$, \ref{dftrl} guarantees the following static regret bound
\begin{align}
\dregret{T}(\bu) \leq 4D\sqrt{\sum_{t=1}^{T}\beta^{2(T-t)}\norm{\bg_t}^2} \,.
\end{align}
\end{theorem}
We next use this result to design an algorithm for nonconvex optimization.

\section{ Discounted-FTRL leads to adaptive nonconvex optimization}
\label{sec:global}

In this section, as a warm-up, let us see the implications of choosing $\alg = \ref{dftrl}$ in \autoref{alg:o2nc}.
First, let us obtain a bound on the expected discounted regret.
By \autoref{thm:dftrl} together with Jensen's inequality, we have the following regret bound for any $t=1,2,\dots, T$:
\begin{align}
\E \left[\dregret{t}(\bu_t) \right] \leq 4D \E \sqrt{\sum_{t=1}^{T}\beta^{2(T-t)}\norm{\bg_t}^2} \leq 4D \sqrt{\sum_{t=1}^{T}\beta^{2(T-t)} \E\norm{\bg_t}^2} \,. 
\end{align}
Since $\E\norm{\bg_t}^2 \leq G^2+\sigma^2$ and $\frac{1}{\sqrt{1-\beta^2}}\leq \frac{1}{\sqrt{1-\beta}}$, it follows that 
\begin{align} \label{exp:regret_bound}
\E\left[\dregret{t}(\bu_t) \right] \leq \frac{4D(G+\sigma)}{\sqrt{1-\beta^2}} \leq  \frac{4D(G+\sigma)}{\sqrt{1-\beta}}\,.  
\end{align}
 
\subsection{From gradient-adaptive regret to nonconvex optimization}

In order to obtain nonconvex optimization guarantees in terms of the $(\lambda,\eps)$-stationarity (\autoref{def:reg_station}), we need to handle the variance term. 
Following {\citep[Lemma 3.2]{zhang2024random}}, the variance term can be bounded as follows.  

\begin{lemma}[Variance bound] \label{lem:variance}
Using the notations of \autoref{lem:o2nc}, for any $t=1,2,\dots, T$, \ref{dftrl} satisfies
\begin{align}  
\E_{t\sim [T]}\E  \norm{\by_t - \barx_t}^2 \leq 12\frac{D^2}{(1-\beta)^2} \,. 
\end{align}
\end{lemma}
\begin{proof}
    From \citep[Lemma 3.2]{zhang2024random},
    it follows that $\E  \sum_{t=1}^T\norm{\by_t - \barx_t}^2 \leq \frac{12}{(1-\beta)^2} \E\sum_{t=1}^T\norm{\bz_t}^2$. Now since $\norm{\bz_t}\leq D$ for all $t=1,2,\dots, T$, after dividing each side by $T$, we get the desired inequality.
\end{proof}

Plugging the regret bound \eqref{exp:regret_bound} into \autoref{lem:o2nc} and combining it with \autoref{lem:variance}, we arrive at the following optimization guarantee in terms of the $(\lambda,\eps)$-stationarity. See \autoref{pf:thm:global} for a proof.

\begin{thmbox}\label{thm:global}
Suppose that $F$ satisfies \autoref{assump} and consider any $\lambda>0$.
For $C >0$, choose $\alg = \ref{dftrl}$ in \autoref{alg:o2nc} with the following parameters:
\begin{align} \label{exp:choice}
\beta = 1-\left(\frac{\eps}{10C}\right)^2, ~~D = \frac{(1-\beta)\eps^{1/2}}{4\lambda^{1/2}},~~\text{and}~~ T=   \frac{1}{1-\beta}\cdot \max\left\{\frac{4 \Delta \lambda^{1/2} }{\eps^{3/2} }, ~\frac{12C}{\eps } \right\}\,.
\end{align}    
Then we have $\E_{t\sim [T]}\E\regnorm{\lambda}{\nabla F(\barx_t)} \leq (1 + \frac{G+\sigma}{C}) \eps$. In other words, a randomly chosen {\bfseries model EMA} $\barx_t$ is a $(\lambda,(1 + \frac{G+\sigma}{C}) \eps)$-stationary point, in expectation.  
\end{thmbox}

\subsection{Optimality and gradient adaptivity}
\label{sec:global_discussion}

Here, we discuss several notable aspects of the guarantee provided in \autoref{thm:global}.

\subsubsection{Optimality}
As shown in \citep[Corollary 5.1]{zhang2024random}, the lower bound on the iteration complexity for finding a $(\lambda,\eps)$-stationary point is $\Omega( (G+\sigma)^2 \Delta\lambda^{1/2} \eps^{-7/2})$, provided that $\lambda\leq  \frac{G^4}{\Delta^2}\eps^{-1}$.
\autoref{thm:global} implies that setting $C = G + \sigma$ achieves this optimal iteration complexity.
 
\begin{corollary} \label{cor:global}
In \autoref{thm:global}, choosing $C =G+\sigma$ leads to the following iteration complexity for finding a $(\lambda,\eps)$-stationary point:
\begin{align}
    O\left(\max\left\{\frac{(G+\sigma)^2\Delta \lambda^{1/2}}{\eps^{7/2}}, ~\frac{(G+\sigma)^3}{\eps^{3}} \right\}\right)\,.
\end{align} 
In particular, treating $G$, $\sigma$, and $\Delta$ as constants, as long as $\lambda \gtrsim \eps$, this leads to the optimal complexity of $O((G+\sigma)^2\Delta\lambda^{1/2}\eps^{-7/2})$.
\end{corollary}

In light of \autoref{lem:convert},   the above optimal  complexity can be converted into the optimal  complexities for smooth settings.

\begin{corollary}[Smooth settings]
 \autoref{cor:global} implies the following optimal iteration complexity for smooth settings.
Choosing $\lambda = O(\eps^{-1})$, it implies the optimal complexity of  $O(\eps^{-4})$ for smooth
loss functions \citep{arjevani2019lower}. Similarly, with $\lambda = O(1)$, 
it achieves the optimal iteration complexitiy of $O(\eps^{-7/2})$ for second-order
smooth loss functions \citep{arjevani2020second}.
\end{corollary}

We next discuss the benefits of using the gradient-adaptive regret bound (\autoref{thm:dftrl}) by considering the case where we do not have knowledge of $G,\sigma$.

\subsubsection{Gradient adaptivity}
A remarkable consequence of \autoref{thm:global} is that, due to the gradient-adaptive regret bound of \autoref{thm:dftrl}, the final convergence guarantee has a better dependence on $G,\sigma$ in the case when we do not have knowledge of them.
For concreteness, in the following discussion, we treat  $G,\sigma,\Delta$ as constants, and focus on the regime  $\lambda \gtrsim \eps$.

First, our \autoref{thm:global} with $C=1$ (since we do not know $G,\sigma$) leads to the following iteration complexity for finding a $(\lambda,\eps)$-stationary point:
\begin{align}
    O\left( (G+\sigma)^{7/2}  \Delta\lambda^{1/2}\eps^{-7/2}\right)
\end{align}
The price we pay for not knowing $G,\sigma$ relative to the lower bound is a multiplicative factor of $(G+\sigma)^{3/2}$.
To see the benefit of this adaptive regret approach,  let us consider the guarantees given by \cite{zhang2024random}.
Their approach is based on choosing online gradient descent for $\alg$ in \autoref{alg:o2nc},  when the learning rate is not properly tuned  with the knowledge of $G$ and $\sigma$, it would lead to the following (suboptimal) discounted regret bound:
\begin{align}  
\E\left[\dregret{t}(\bu_t) \right]   \leq  O\left(\frac{D(G+\sigma)^2}{\sqrt{1-\beta}}\right)\,.  
\end{align}
Then, the resulting iteration complexity becomes
$O(\Delta \lambda^{1/2} (\frac{\eps}{(G+\sigma)^2})^{-7/2})$, which is equal to $O\left( (G+\sigma)^7\Delta \lambda^{1/2} \eps^{-7/2}\right)$.   
This is larger than the complexity due to our adaptive approach by a multiplicative factor of $(G+\sigma)^{7/2}$.

Next, we build on the results from this section and consider a better approach to design an adaptive nonconvex optimizer.

\section{Coordinate-wise adaptivity via (clipped-)Adam}
\label{sec:coordinate}

In this section, we consider the setting where the Lipschitzness constants vary  across different coordinates, which is  empirically observed to be reflective of practical neural network training (see, \emph{e.g.} \citep{crawshaw2022robustness,zhuang2022understanding}).
Formally, we consider the following setting.

\begin{assumption} \label{assump_coordinate}
Under the same setting as \autoref{assump}, we replace the last two conditions with the following coordinate-wise version:
\begin{itemize}
\item For each coordinate $i=1,2,\dots, d$, there is a Lipschitzness constant $G_i>0$ and a variance constant $\sigma_i>0$ such that $\forall \bx$, $\left|\partial_i F (\bx) \right| \leq G_i$ and the stochastic gradient $\bg\gets \stograd(\bx,r)$ satisfies $\E[\bg[i]] =\partial_i F(\bx_i)$ and $\E\abs{\bg[i] - \partial_i F (\bx)}^2 \leq \sigma_i^2$. (Here, $\partial_i F$ denotes the partial derivative of $F$ w.r.t. the $i$-th coordinate.)
\end{itemize}
Let $\bG \coloneqq (G_1,\dots, G_d)$ and $\bsigma \coloneqq (\sigma_1,\dots, \sigma_d)$.
Then, the above condition implies the last two conditions in \autoref{assump} with $G= \norm{\bG}_2$ and $\sigma = \norm{\bsigma}_2$. 
\end{assumption}

As we mentioned before, the previous approaches \citep{cutkosky2023optimal,zhang2024random} choose the online learner $\alg$ to be online gradient descent, and hence choosing the learning rate requires the knowledge of $G_i,\sigma_i$ for all $i$.
However, for neural network training, $d$ is equal to the number of parameters in the network, so tuning them individually is computationally infeasible.
We instead consider running \ref{dftrl} {\bf coordinate-wise} in \autoref{alg:o2nc}, which will automatically adapt to each coordinate.
We begin with an important observation that such an approach in fact leads to a popular optimizer widely used in practice.

\subsection{Coordinate-wise discounted FTRL corresponds to (clipped-)Adam}

For notational simplicity, fix a coordinate among $i=1,2,\dots, d$, and let us denote the iterate by $x_t$, the stochastic gradient by $g_t$, and the update by $z_t$.
Then the resulting optimizer becomes:
\begin{align}\tag{\textsc{clipped-Adam}}\label{adam}
\boxed{ z_{t+1}= -\clip_D\left(D \frac{\sum_{s=1}^{t} \beta_1^{t-s}g_{s}}{\sqrt{\sum_{s=1}^{t} \beta_2^{t-s}g_s^2}}\right)\,,}   
\end{align}
where $\beta_1 =\beta$ and $\beta_2 =\beta^2$.
Here, again if the denominator is zero, \emph{i.e.}, if $g_1=\dots = g_t=0$, then we set the update to be zero, \emph{i.e.}, $z_{t+1}=0$.
Note that {\bf this is almost exactly the Adam optimizer}~\citep{kingma2015adam}, except that now we add clipping to control the variance of the iterates relative to their EMA. 
Notably, \ref{adam} retains one of the most important properties of Adam: it is \emph{scale-invariant}. 
The scale invariance causes the optimizer to make updates of the same magnitude on each coordinate even when the scale differs across different coordinates.

In practice, we expect that the clipping operation will effectively be a no-op. This is because, when the algorithm is converging (even if the convergence is somewhat slow), the gradients are likely to behave as approximately mean-zero random variables (due to factors such as stochastic noise, unstable training trajectories, etc.). In such cases, standard concentration inequalities imply that $\sum_{s=1}^t \beta^{t-s}g_s \lesssim \sqrt{\sum_{s=1}^t (\beta^{t-s}g_s)^2}$, and hence, the clipping has no effect.

We also remark that \ref{adam} does not consider the ``bias correction'' terms in the original updates of Adam~\citep{kingma2015adam}. However, note that the bias correction terms are coordinate-independent, and they can be merged into the scalar $D$. 

\subsection{Nonconvex optimization guarantees of \ref{adam}}

We next discuss the theoretical guarantees of \ref{adam} for nonconvex and nonsmooth optimizaton.
Inspired by \citep{duchi10adagrad,mcmahan2010adaptive}, where the coordinate-wise online learners lead to regret bounds with respect to the $L_1$ norms of stochastic gradients, we consider the following variant of \autoref{def:reg_station}, in the same vein as \citep[Section 4]{cutkosky2023optimal}.

\begin{definition} [{\bf $(\lambda,\eps)$-$L_1$-stationary point}] \label{def:reg_station_coordinate}
Suppose $F:\R^d\to \R$ is differentiable.
We say $\bx$ is a  $(\lambda,\eps)$-$L_1$-stationary point of $F$ if $\regnorm{\lambda}{\nabla F(\bx)}_1 \leq \eps$, where 
\begin{align}
\regnorm{\lambda}{\nabla F(\bx)}_{1} \coloneqq \inf_{\substack{p \in \mathcal{P}(\R^d),\\\E_{\by\sim p}[\by] =\bx} } \left\{\norm{\E[\nabla F(\by)]}_1 +\lambda \cdot \E\norm{\by-\bx}_2^2 \right\}\,.
\end{align}
\end{definition}

Using the fact  $\norm{\cdot}_1 \leq \sqrt{d} \norm{\cdot}_2$, one can connect the two notions of $(\lambda,\eps)$-stationary points.
\begin{lemma} \label{lem:two_stationary}
    A $({\lambda}/{\sqrt{d}}, {\eps}/{\sqrt{d}})$-stationary point is a $(\lambda,\eps)$-$L_1$-stationary point. 
\end{lemma} 

In order to obtain the guarantee in terms of $L_1$-norm, we consider the coordinate-wise version of discounted-to-online conversion, in the same vein as \citep[Appendix G]{cutkosky2023optimal}. See \autoref{pf:lem:o2nc_coordinate} for details.

\begin{lemma}[{\bf $L_1$-variant of \autoref{lem:o2nc}}]   \label{lem:o2nc_coordinate}
Suppose that $F$ satisfies \autoref{assump_coordinate}. 
Consider the comparator sequence chosen as  $\bu_t$ defined as  $\bu_t[i] \coloneqq -D \frac{\sum_{s=1}^t\beta^{-s}\partial_i F(\bx_s)}{\left|\sum_{s=1}^t\beta^{-s}\partial_i F(\bx_s)\right|}$ for  $i=1,2,\dots, d$. 
Then, \autoref{alg:o2nc} gives
\begin{align}
\E_{t\sim [T]}\E\norm{\E_{\by_t}\nabla F(\by_t)}_1 &\leq \frac{\Delta}{DT} + \frac{2\norm{\bG + \bsigma}_1}{(1-\beta) T} +\norm{\bsigma}_1 \sqrt{1-\beta}   \\
&\quad +\frac{1}{DT} \left[ \beta\cdot  \E \left[\dregret{T}(\bu_T)\right]  + (1-\beta)\cdot  \sum_{t=1}^T   \E \left[\dregret{t}(\bu_t)\right] \right]\,,
\end{align}
where  $\by_t$ is distributed over $\{\bx_s\}_{s=1}^t$ as $\P(\by_t = \bx_s) = \beta^{t-s} \cdot \frac{1-\beta}{1-\beta^t}$ for $s=1,2,\dots, t$.
\end{lemma}

Next, let us consider the (expected) regret bound.
Fix a coordinate $i=1,\dots,d$.
Then, by the one-dimensional version of \autoref{thm:dftrl} together with Jensen's inequality, we have the following regret bound for any $t=1,2,\dots, T$:
\begin{align} 
\E \left[\dregret{t}(\bu_t[i]) \right] \leq 4D  \sqrt{\sum_{t=1}^{T}\beta^{2(T-t)}\E\abs{\bg_t[i]}^2} \leq \frac{4D(G_i+\sigma_i)}{\sqrt{1-\beta}}\,. 
\end{align} 
Hence, taking the sum over all coordinates $i=1,\dots,d$, we obtain
\begin{align}
\label{exp:regret_bound_coordinate}
\E \left[\dregret{t}(\bu_t) \right]  \leq \frac{4D\norm{\bG+\bsigma}_1}{\sqrt{1-\beta}}\,. 
\end{align}
Combining these together, one get the following guarantee in terms of the $L_1$ norm. See \autoref{pf:thm:coordinate} for a proof.

\begin{thmbox}\label{thm:coordinate}
 Suppose that $F$ satisfies \autoref{assump_coordinate} and consider any $\lambda>0$.
For $C>0$, choose the coordinate-wise optimizer \ref{adam} in \autoref{alg:o2nc} with the following parameters:
\begin{align}  
\beta = 1-\left(\frac{\eps}{10C}\right)^2, ~~D = \frac{(1-\beta)\eps^{1/2}}{4d^{1/2}\lambda^{1/2}},~~\text{and}~~ T=    \frac{1}{1-\beta}\cdot \max\left\{ \frac{4\Delta d^{1/2} \lambda^{1/2}}{\eps^{3/2}}, ~\frac{12C}{\eps} \right\}\,.
\end{align}    
Then we have $\E_{t\sim [T]}\E\regnorm{\lambda}{\nabla F(\barx_t)}_1 \leq (1 + \frac{\norm{\bG+\bsigma}_1}{C}) \eps$. In other words, a randomly chosen {\bfseries model EMA} $\barx_t$ is a $(\lambda,(1 + \frac{\norm{\bG+\bsigma}_1}{C}) \eps)$-$L_1$-stationary point, in expectation.  
\end{thmbox}

\subsection{Benefits of coordinate-wise adaptivity of \ref{adam}}
\label{sec:coordinate_discussion}

In this section, we discuss the benefits of coordinate-wise adaptivity  by examining the guarantee from \autoref{thm:coordinate} and compare it with that of \autoref{thm:global}.
We begin with the $(\lambda,\eps)$-$L_1$-stationary point guarantee due to \autoref{thm:coordinate}.
We consider the scenario where $\beta$ is carefully tuned by making the optimal choice of $C$. 

\begin{corollary} \label{cor:coordinate}
In \autoref{thm:coordinate}, choosing $C =\norm{\bG+\bsigma}_1$ leads to the following iteration complexity for finding a $(\lambda,\eps)$-$L_1$-stationary point:
\begin{align} \label{complexity:coordinate_formal}
    O\left(\max\left\{\frac{\norm{\bG+\bsigma}_1^2\Delta d^{1/2}\lambda^{1/2}}{\eps^{7/2}}, ~\frac{\norm{\bG+\bsigma}_1^3}{\eps^{3}} \right\}\right)\,.
\end{align} 
\end{corollary}

In order to better appreciate the benefits of coordinate-wise adaptivity, let us compare the above iteration complexity with that of \autoref{thm:global}.

For concreteness,  we treat $G=\norm{\bG}_2$ and $\sigma = \norm{\bsigma}_2$ as constants throughout, and more importantly, we assume that the {\bf coordinates are heterogeneous} in the sense that 
\begin{align} \label{cond:hetero}
    \norm{\bG+\bsigma}_1 \approx \norm{\bG+\bsigma}_2\,.
\end{align}
The assumption \eqref{cond:hetero} roughly says that a few coordinates of $\bG +\bsigma$ take much larger values than the rest; if all the coordinates of $\bG +\bsigma$ have similar magnitudes, then $\norm{\bG+\bsigma}_1 \approx \sqrt{d}\norm{\bG+\bsigma}_2$.
In the case $\lambda \gtrsim \eps$, \autoref{cor:coordinate} implies that the iteration complexity is 
\begin{align} \label{complexity:coordinate}
    O(\norm{\bG+\bsigma}_1^2\Delta d^{1/2}\lambda^{1/2}\eps^{-7/2})\,.
\end{align}
Next, let us consider the counterpart that does not adapt to each coordinate separately. In this case, we apply \autoref{lem:two_stationary}, which tells us that to find a $(\lambda, \eps)$-$L_1$ stationary point it suffices to find a $({\lambda}/{\sqrt{d}}, {\eps}/{\sqrt{d}})$-stationary point.
Then, from \autoref{cor:global}, the iteration complexity is $O(\norm{\bG+\bsigma}_2^2\Delta(\lambda/\sqrt{d})^{1/2}(\eps/\sqrt{d})^{-7/2})$, \emph{i.e.},
\begin{align} \label{complexity:global}
    O(\norm{\bG+\bsigma}_2^2\Delta d^{3/2}\lambda^{1/2}\eps^{-7/2})\,.
\end{align}
Hence, when \eqref{cond:hetero} holds, \eqref{complexity:coordinate} can be lower than \eqref{complexity:global} by a multiplicative factor of $d$, showing the benefits of coordinate-wise adaptivity.

\section{Discussion}
\label{sec:discussion}

Our analyses of Adam based on the discounted-to-online conversion is quite different than the previous ones.
As discussed in \autoref{sec:related}, the previous analyses often result in guarantees that are not quite reflective of practice---\emph{e.g.}, the rates get better without momentum and the rates are no better than that of non-adaptive methods.
In contrast, our analyses and results highlight the role of the practical components as highlighted below.

\begin{itemize}
\item {\bf Momentum.} In order to obtain a low discounted regret, any sensible online learner should integrate the past history of stochastic gradients $\bg_{1:t}$ to make the decision $\bz_{t+1}$. 
Such online learners under the discounted-to-online conversion lead to momentum methods that integrate $\bg_{1:t}$ to obtain the next increment $\bz_{t+1}$. 
In particular, non-momentum methods would correspond to aggressive online learners that only use  the last gradient $\bg_t$ to make the decision $\bz_{t+1}$. 
This perspective provides new insights into understanding the role of momentum, as echoed by \cite{ahn2024understanding}.

\item {\bf Adaptive learning rates.} The adaptive learning rate due to \ref{dftrl} leads to a gradient-adaptive regret bound (\autoref{thm:dftrl}), which is important to obtain a better Lipshitzness dependence (\autoref{sec:global_discussion}) as well as the coordinate-wise adaptivity for high-dimension settings (\autoref{sec:coordinate_discussion}). 
Our analysis offers theoretical benefits of adaptive learning rate from a discounted regret perspective. 

\item {\bf Model EMA.} Lastly, the discounted-to-nonconvex conversion (\autoref{alg:o2nc}) naturally leads to guarantees in terms of the model EMA, $\barx_t$. At a high level (see \autoref{pf:lem:o2nc} precise details), this is because for a dynamic environment, it is important to discount the losses such that online learners adapt to changing environments.
The appearance of model EMA in the discounted-to-nonconvex conversion provides a new perspective on its role.
\end{itemize}

Our analyses and results have several limitations and raise several interesting questions. Firstly, \ref{adam} does not precisely match the original Adam algorithm, warranting further investigation into the original Adam update. Specifically, our analysis suggests choosing $\beta_1=\beta$ and $\beta_2=\beta^2$, which does not align with the commonly used practical choices. Understanding the exact roles of these practical choices for $\beta_1$ and $\beta_2$ would be valuable.

In \autoref{sec:global_discussion}, we observed that our iteration complexity for finding a $(\lambda,\eps)$-stationary point is $O (\Delta (G+\sigma)^{7/2} \lambda^{1/2} \eps^{-7/2})$ when $G$ and $\sigma$ are unknown. Investigating whether this complexity is optimal presents another intriguing direction for future research.

Lastly, from a practical standpoint, developing a more advanced online learner for discounted regret and designing an algorithm that surpasses Adam in practicality would have significant practical implications.

\subsection*{Funding Acknowledgments}
AC is supported by NSF grant number CCF-2211718.

\bibliography{ref}
\bibliographystyle{plainnat}

\newpage 
\appendix
\renewcommand{\appendixpagename}{\centering \LARGE Appendix}
\appendixpage
\startcontents[section]
\printcontents[section]{l}{1}{\setcounter{tocdepth}{2}}

\section{Proof of discounted-to-nonconvex conversion (\autoref{lem:o2nc})}
\label{pf:lem:o2nc}
Note first that via a change of summation, we get 
\begin{align}
\sum_{n=1}^T \sum_{t=1}^n \beta^{n-t} (1-\beta) (F(\bx_t)- F(\bx_{t-1})) &= 
\sum_{t=1}^T \sum_{n=t}^T \beta^{n-t} (1-\beta) (F(\bx_t)- F(\bx_{t-1})) \\
&= \sum_{t=1}^T    (1-\beta^{T-t+1}) (F(\bx_t)- F(\bx_{t-1}))\\
&= F(\bx_T) -F(\bx_0) - \sum_{t=1}^T \beta^{T-t+1} (F(\bx_t) -F(\bx_{t-1}))\,.
\end{align}
Rearranging the above together with the fact $F(\bx_0)-F(\bx_T) \leq F(\bx_0) -\inf_{\bx} F(\bx)=: \Delta$, we get 
\begin{align}
-\Delta \leq \underbrace{\E \left[ \sum_{n=1}^T \sum_{t=1}^n \beta^{n-t} (1-\beta) (F(\bx_t)- F(\bx_{t-1}))
\right]}_{\circled{A}} + \underbrace{\E \left[ \sum_{t=1}^T  \beta^{T-t+1} (F(\bx_t) -F(\bx_{t-1})) \right]}_{\circled{B}}\,.
\end{align}
We also recall the following fact about exponential random variable due to \citep[Lemma 3.1]{zhang2024random}.
\begin{lemma}\label{lem:exp}
    Let $\scale \sim \text{Exp}(\lambda)$ for some $\lambda>0$, then
    \begin{align}
        \E_\scale [F(\bx +\scale \bz) - F(\bx)] = \E_{\scale}[\inp{\nabla F(\bx +\scale \bz)}{\bz}]/\lambda\,.
    \end{align}
\end{lemma}

By \autoref{lem:exp} with $\lambda = 1$, together with $\bx_t = \bx_{t-1} + \scale_t \bz_t$, it follows that 
$$
   \E [ F(\bx_t) - F(\bx_{t-1})] = \E\left[ \inp{\bg_t}{\bz_t}\right].
$$
This identity indicates that the function gap is exactly equal to the linearization of the function gap. In this sense, the randomization renders the first-order Taylor approximation perfectly accurate.

Now, with this result, we will address each term separately.

\subsubsection*{\underline{Analysis of  \circled{A}}}

Note that for each $t\leq n$, since $\bx_t = \bx_{t-1} + \scale_t \bz_t$ for  $\scale_t  {\sim} \text{Exp}(1)$,  \autoref{lem:exp} yields
\begin{align}
\E[F(\bx_t) - F(\bx_{t-1})] &= \E\inp{\nabla F(\bx_t
)}{\bz_t } =  \E\inp{\nabla F(\bx_t
)}{\bu_n}+ \E\inp{\nabla F(\bx_t
)}{\bz_t -\bu_n}   \\ 
&= \E\inp{\nabla F(\bx_t
)}{\bu_n}+\E\inp{\nabla F(\bx_t
) - \bg_t}{\bz_t -\bu_n } + \E\inp{\bg_t}{\bz_t- \bu_n}  \\
&= \underbrace{\E\inp{\nabla F(\bx_t
)}{\bu_n}}_{\circled{1}}+\underbrace{\E\inp{\nabla F(\bx_t
) - \bg_t}{-\bu_n }}_{\circled{2}} + \underbrace{\E\inp{\bg_t}{\bz_t- \bu_n}}_{\circled{3}},
\end{align}
where the last line follows from the fact $\E[\inp{\nabla F(\bx_t
) - \bg_t}{\bz_t}]=0$.
More specifically, note that the randomness in the stochastic gradient oracle is independent of the randomness due to $\alpha_t$. Since $\mathbb{E}[\nabla F(\bx_t)- \bg_t] =0$, it follows that
$$\mathbb{E}[\langle \nabla F(\bx_t)- \mathbf{\bg_t}, \mathbf{\bz_t} \rangle]   = \mathbb{E}[ \langle \mathbb{E}[ \nabla F(\bx_t)- \mathbf{\bg_t}], \mathbf{\bz_t} \rangle] = 0,$$
where the inner expectation is with respect to the randomness in the stochastic gradient oracle and the outer is with respect to all other quantities.

Now let us handle each term. 
\begin{itemize}[leftmargin=25pt]
\item[\circled{1}:]  Note that using the definition of $\by_n$, we have
\begin{align}
&\E\sum_{t=1}^n \beta^{n-t} (1-\beta) \inp{\nabla F(\bx_t
)}{\bu_n}  =  (1-\beta)\E \inp{\sum_{t=1}^n \beta^{n-t}\nabla F(\bx_t
)}{-D\frac{\sum_{t=1}^n \beta^{n-t}\nabla F(\bx_t
)}{\norm{\sum_{t=1}^n \beta^{n-t}\nabla F(\bx_t
)}} } \\
&\quad = (1-\beta^n) \E \inp{\sum_{t=1}^n \frac{1-\beta}{1-\beta^{n}}\beta^{n-t}\nabla F(\bx_t
)}{-D\frac{\sum_{t=1}^n \frac{1-\beta}{1-\beta^{n}}\beta^{n-t}\nabla F(\bx_t
)}{\norm{\sum_{t=1}^n \frac{1-\beta}{1-\beta^{n}}\beta^{n-t}\nabla F(\bx_t
)}} }\\
&\quad =  -D (1-\beta^n)  \E \norm{\E_{\by_n}\nabla F(\by_n)}\leq -D    \E\norm{\E_{\by_n}\nabla F(\by_n)} + DG\beta^n\,.
\end{align}
Therefore, summing over $n=1,\dots, T$, we obtain:
\begin{align}
\E\sum_{n=1}^T\sum_{t=1}^n \beta^{n-t} (1-\beta) \inp{\nabla F(\bx_t
)}{\bu_n}  \leq -D \E \sum_{t=1}^T \norm{\E_{\by_t}\nabla F(\by_t)} + \frac{DG}{1-\beta}\,.
\end{align}
\item[\circled{2}:] For the second term, using Cauchy-Schwartz inequality, we have
\begin{align}
\E\sum_{t=1}^n \beta^{n-t}   \inp{\nabla F(\bx_t
) - \bg_t}{-\bu_n} \leq \sqrt{ \E \norm{\sum_{t=1}^n \beta^{n-t}   (\nabla F(\bx_t
) - \bg_t)}^2 \E\norm{\bu_n}^2 }\,.
\end{align}
Using the bounded variance assumption on the stochastic gradient oracle, we have
\begin{align}
\E \norm{\sum_{t=1}^n \beta^{n-t}   (\nabla F(\bx_t
) - \bg_t)}^2 = \E \sum_{t=1}^n \beta^{2(n-t)}   \norm{\nabla F(\bx_t
) - \bg_t}^2  \leq \frac{\sigma^2}{1-\beta^2}\,.
\end{align}
Therefore, summing over $n=1,\dots, T$, and using the fact that $\frac{1}{1-\beta^2}\leq \frac{1}{1-\beta}$, we get the following bound on the second term:
\begin{align}
\E\sum_{n=1}^T\sum_{t=1}^n \beta^{n-t} (1-\beta)  \inp{\nabla F(\bx_t
) - \bg_t}{-\bu_n} \leq \sum_{n=1}^T (1-\beta) \cdot \frac{\sigma D}{\sqrt{1-\beta^2}} \leq \sigma D T \sqrt{1-\beta}\,. 
\end{align}
\item[\circled{3}:] Lastly, for the third term, we have
\begin{align}
\E\sum_{n=1}^T\sum_{t=1}^n \beta^{n-t} (1-\beta) \inp{\bg_t}{\bz_t- \bu_n} &= (1-\beta) \E\sum_{n=1}^T\left[\sum_{t=1}^n  \E\inp{\beta^{n-t}\bg_t}{\bz_t- \bu_n}\right]\\
&= (1-\beta) \E \sum_{t=1}^T \dregret{t}(\bu_t) \,.
\end{align}
\end{itemize}

\subsubsection*{\underline{Analysis of  \circled{B}}}

\noindent Note that for each $t$, since $\bx_t \gets \bx_{t-1} + \scale_t \bz_t$ for  $\scale_t  {\sim} \text{Exp}(1)$,  \autoref{lem:exp} yields
\begin{align}
\E[F(\bx_t) - F(\bx_{t-1})] &= \E\inp{\nabla F(\bx_t
)}{\bz_t } = \E\inp{\bg_t}{\bz_t }\\
&= \E\inp{\bg_t}{\bz_t -\bu_T} + \E\inp{\bg_t}{\bu_T}  \leq \E\inp{\bg_t}{\bz_t -\bu_T} + D(G+\sigma)\,. 
\end{align}
Thus,  
\begin{align}
\E \left[ \sum_{t=1}^T  \beta^{T-t+1} (F(\bx_t) -F(\bx_{t-1})) \right] &= \beta \E \sum_{t=1}^T \left[\inp{\beta^{T-t}\bg_t}{\bz_t-\bu_T} + \beta^{T-t} D(G+\sigma) \right]  \\
&\leq \beta \E [\dregret{T}(\bu_T)] + \frac{D(G+\sigma)}{1-\beta}\,.
\end{align}

\subsubsection*{\underline{Combining  \circled{A} and \circled{B}}}

Combining the above analyses and rearranging, it follows that
\begin{align}
D \E \sum_{t=1}^T   \norm{\E_{\by_t}\nabla F(\by_t)}  &\leq \Delta + \frac{DG}{1-\beta} + \sigma DT\sqrt{1-\beta} + (1-\beta) \E \sum_{t=1}^T \left[ \dregret{t}(\bu_t)\right] \\
&\quad + \beta \E [\dregret{T}(\bu_T)] + \frac{D(G+\sigma)}{1-\beta}\,.
\end{align}
Dividing both sides by $DT$, we get the desired result.

\subsection{Proof of the coordinate-wise version (\autoref{lem:o2nc_coordinate})}
\label{pf:lem:o2nc_coordinate}

The proof closely follows that of \autoref{lem:o2nc}.
In particular,  with $\Delta \coloneqq F(\bx_0) -\inf_{\bx} F(\bx)$, we have  
\begin{align}
-\Delta \leq \underbrace{\E \left[ \sum_{n=1}^T \sum_{t=1}^n \beta^{n-t} (1-\beta) (F(\bx_t)- F(\bx_{t-1}))
\right]}_{\circled{A}} + \underbrace{\E \left[ \sum_{t=1}^T  \beta^{T-t+1} (F(\bx_t) -F(\bx_{t-1})) \right]}_{\circled{B}}\,.
\end{align}

We begin with the term {\small \circled{B}}.
Using the same decomposition as before, we have
\begin{align}
\E[F(\bx_t) - F(\bx_{t-1})] 
&= \E\inp{\bg_t}{\bz_t -\bu_T} + \E\inp{\bg_t}{\bu_T} = \E\inp{\bg_t}{\bz_t -\bu_T} + \sum_{i=1}^d \E \bg_t[i]\bu_T[i] \\
&\leq \E\inp{\bg_t}{\bz_t -\bu_T} + D \sum_{i=1}^d (G_i+\sigma_i)\,. 
\end{align}
Thus,  
\begin{align}
\E \left[ \sum_{t=1}^T  \beta^{T-t+1} (F(\bx_t) -F(\bx_{t-1})) \right] &= \beta \E \sum_{t=1}^T \left[\inp{\beta^{T-t}\bg_t}{\bz_t-\bu_T} + \beta^{T-t} \sum_{i=1}^d D(G_i+\sigma_i) \right]  \\
&\leq \beta \E [\dregret{T}(\bu_T)] + \frac{D\sum_{i=1}^d(G_i+\sigma_i)}{1-\beta}\,.
\end{align}

Moving onto the term {\small \circled{A}}, we again use the same decomposition:
\begin{align}
\E[F(\bx_t) - F(\bx_{t-1})]  
&= \underbrace{\E\inp{\nabla F(\bx_t
)}{\bu_n}}_{\circled{1}}+\underbrace{\E\inp{\nabla F(\bx_t
) - \bg_t}{-\bu_n }}_{\circled{2}} + \underbrace{\E\inp{\bg_t}{\bz_t- \bu_n}}_{\circled{3}}\,.
\end{align} 

As before, let us handle each term one by one separately. 
\begin{itemize}[leftmargin=25pt]
\item[\circled{1}:]  Note that using the definition of $\by_n$, for each coordinate $i=1,\dots, d$, we have
\begin{align}
&\E\sum_{t=1}^n \beta^{n-t} (1-\beta) \partial_i F(\bx_t
) \bu_n[i]  =  (1-\beta) \E \left[\left(\sum_{t=1}^n \beta^{n-t}\nabla F(\bx_t
) \right) \left( -D\frac{\sum_{t=1}^n \beta^{n-t}\partial_i F(\bx_t
)}{\abs{\sum_{t=1}^n \beta^{s-t}\partial_i F(\bx_t
)}} \right)\right]\\
&\quad = (1-\beta^n) \E\left[\left(\sum_{t=1}^n \frac{1-\beta}{1-\beta^{n}}\beta^{n-t}\nabla F(\bx_t
)\right) \left(-D\frac{\sum_{t=1}^n \frac{1-\beta}{1-\beta^{n}}\beta^{n-t}\nabla F(\bx_t
)}{\abs{\sum_{t=1}^n \frac{1-\beta}{1-\beta^{n}}\beta^{n-t}\nabla F(\bx_t
)}} \right)\right]\\
&\quad =  -D (1-\beta^n) \E  \abs{\E_{\by_n}\partial_i F(\by_n)}\leq -D    \E\abs{\E_{\by_n}\partial_i F(\by_n)} + DG_i\beta^n\,.
\end{align}
Therefore, summing over $i=1,\dots, d$ and then $n=1,\dots, T$, we obtain:
\begin{align}
\E\sum_{n=1}^T\sum_{t=1}^n \beta^{n-t} (1-\beta) \inp{\nabla F(\bx_t
)}{\bu_n}  \leq -D \E \sum_{t=1}^T \norm{\E_{\by_t}\nabla F(\by_t)}_1 + \frac{D \sum_{i=1}G_i}{1-\beta}\,.
\end{align}
\item[\circled{2}:]  For each coordinate $i=1,\dots, d$, we have
\begin{align}
\E\sum_{t=1}^n \beta^{n-t}   (\partial_i F(\bx_t
) - \bg_t[i]) (-\bu_n[i])  \leq \sqrt{ \E \abs{\sum_{t=1}^n \beta^{n-t}   (\partial_i F(\bx_t
) - \bg_t[i])}^2 \E\abs{\bu_n[i]}^2 }\,.
\end{align}
Using the coordinate-wise bounded variance assumption on the stochastic gradient oracle, 
\begin{align}
\E \abs{\sum_{t=1}^n \beta^{n-t}   (\partial_i F(\bx_t
) - \bg_t[i])}^2 = \E \sum_{t=1}^n \beta^{2(n-t)}   \abs{\partial_i F(\bx_t
) - \bg_t[i]}^2  \leq \frac{\sigma_i^2}{1-\beta^2}\,.
\end{align}
Therefore, summing over $n=1,\dots, T$, and using the fact that $\frac{1}{1-\beta^2}\leq \frac{1}{1-\beta}$, we get the following bound on the second term:
\begin{align}
\E\sum_{n=1}^T\sum_{t=1}^n \beta^{n-t} (1-\beta)  \inp{\nabla F(\bx_t
) - \bg_t}{-\bu_n} &\leq \sum_{n=1}^T (1-\beta) \cdot \frac{ D\sum_{i=1}^d \sigma_i}{\sqrt{1-\beta^2}} \\
&\leq D T \left(\sum_{i=1}^d\sigma_i\right)\sqrt{1-\beta}\,. 
\end{align}
\item[\circled{3}:] We use the same manipulation as before:
\begin{align}
\E\sum_{n=1}^T\sum_{t=1}^n \beta^{n-t} (1-\beta) \inp{\bg_t}{\bz_t- \bu_n}  &= (1-\beta) \E \sum_{t=1}^T \dregret{t}(\bu_t) \,.
\end{align}
\end{itemize}
Combining the above, we get the desired result in \autoref{lem:o2nc_coordinate}.

\section{Proof of main theorems}

\subsection{Proof of \autoref{thm:global}} 
\label{pf:thm:global}
By \autoref{def:reg_station}, since $\E[\by_t] = \barx_t$, it holds that
\begin{align}
\E_{t\sim [T]}\regnorm{\lambda}{\nabla F(\barx_t) } \leq    \E_{t\sim [T]}\left[\norm{\E_{\by_t}\nabla F(\by_t)} + \lambda  \E_{\by_t} \norm{\by_t - \barx_t}^2 \right]\,. 
\end{align}

We begin with the second term (the variance term). By \autoref{lem:variance}, we have
\begin{align}
\lambda \E_{t\sim [T]}\E_{\by_t} \norm{\by_t - \barx_t}^2 &\leq   12\frac{\lambda D^2}{(1-\beta)^2}\,. 
\end{align}
Hence, by choosing $D= \frac{(1-\beta)\eps^{1/2}}{4\lambda^{1/2}}$, it follows that $\lambda \cdot \E_{t\sim [T]}\E_{\by_t} \norm{\by_t - \barx_t}^2 \leq \eps$.

Next consider the first term (the norm of the averaged gradients). 
Plugging the regret bound \eqref{exp:regret_bound} into \autoref{lem:o2nc}, we get
\begin{align}
\E_{t\sim [T]}\norm{\E_{\by_t}\nabla F(\by_t)}&\leq  \frac{\Delta}{DT} + \frac{2G+\sigma}{(1-\beta) T} +\sigma \sqrt{1-\beta}    +     \frac{4(G+\sigma)}{T\sqrt{1-\beta}}  +   4(G+\sigma) \sqrt{1-\beta}  \\
&\leq  \frac{4\Delta \lambda^{1/2}}{ (1-\beta)\eps^{1/2}T} + \frac{6G+5\sigma}{(1-\beta) T} +(4G+5\sigma) \sqrt{1-\beta}    \,,
\end{align}
where the last line follows since $ \frac{1}{\sqrt{1-\beta}} \leq \frac{1}{1-\beta}$ and $D= \frac{(1-\beta)\eps^{1/2}}{4\lambda^{1/2}}$.
Choosing $\beta = 1-(\frac{\eps}{10C})^2$, the last term is bounded by $\frac{G+\sigma}{2C} \eps$. Moreover, choosing $T= (1-\beta)^{-1}\cdot \max\left\{4\Delta \lambda^{1/2}\eps^{-3/2} , ~12C\eps^{-1}  \right\}$, the first and second terms are bounded by $\eps$ and $\frac{G+\sigma}{2C}\eps$, respectively. This concludes the proof.  

\subsection{Proof of \autoref{thm:coordinate}}
\label{pf:thm:coordinate}

By \autoref{def:reg_station_coordinate}, since $\E[\by_t] = \barx_t$, it holds that
\begin{align}
\E_{t\sim [T]}\regnorm{\lambda}{\nabla F(\barx_t) }_1 \leq    \E_{t\sim [T]}\left[\norm{\E_{\by_t}\nabla F(\by_t)}_1 + \lambda  \E_{\by_t} \norm{\by_t - \barx_t}_2^2 \right]\,. 
\end{align}

We begin with the second term (the variance term). This time, given that now each coordinate of update $\bz_t$ is bounded by $D$, \emph{i.e.}, $\abs{\bz_t[i]}\leq D$, applying the variance bound due to \autoref{lem:variance} coordinate-wise implies:
\begin{align}  
\lambda \E_{t\sim [T]}\E_{\by_t} \norm{\by_t - \barx_t}_2^2 \leq 12\frac{\lambda d D^2}{(1-\beta)^2} \,. 
\end{align} 

Hence, by choosing $D= \frac{(1-\beta)\eps^{1/2}}{4d^{1/2}\lambda^{1/2}}$, it follows that $\lambda \E_{t\sim [T]}\E_{\by_t} \norm{\by_t - \barx_t}^2 \leq \eps$.

Next consider the first term (the $L_1$-norm of the averaged gradients). 
Plugging the regret bound \eqref{exp:regret_bound_coordinate} into \autoref{lem:o2nc_coordinate}, and doing similar manipulations as the proof of \autoref{thm:global}, we get
\begin{align}
\E_{t\sim [T]}\norm{\E_{\by_t}\nabla F(\by_t)}_1 &\leq \frac{\Delta}{DT} + \frac{\norm{6\bG+5\bsigma}_1}{(1-\beta) T} +\norm{4\bG+5\bsigma}_1 \sqrt{1-\beta}    \\
&= \frac{4\Delta d^{1/2}\lambda^{1/2} }{  (1-\beta)\eps^{1/2}T} + \frac{\norm{6\bG+5\bsigma}_1}{(1-\beta) T} +\norm{4\bG+5\bsigma}_1 \sqrt{1-\beta}  \,,
\end{align} 
where the last line follows since $D=\frac{(1-\beta)\eps^{1/2}}{4d^{1/2}\lambda^{1/2}}$.
Choosing $\beta = 1-(\frac{\eps}{10C})^2$, the last term is bounded by $\frac{\norm{\bG+\bsigma}_1}{2C}\cdot \eps$. Moreover, choosing $T= (1-\beta)^{-1}\cdot \max\left\{4\Delta d^{1/2} \lambda^{1/2}\eps^{-3/2}, ~12C\eps^{-1} \right\}$, the first and second terms are bounded by $\eps$ and $\frac{\norm{\bG+\bsigma}_1}{2C}\cdot \eps$, respectively. This concludes the proof.  
 
\newpage
\section*{NeurIPS Paper Checklist}

\begin{enumerate}

\item {\bf Claims}
\item[] Question: Do the main claims made in the abstract and introduction accurately reflect the paper's contributions and scope?
\item[] Answer: \answerYes{} 
\item[] Justification: The abstract and introduction accurately reflect the paper's contributions regarding the analysis of Adam with EMA for nonconvex and nonsmooth optimization. 
\item[] Guidelines:
\begin{itemize}
\item The answer NA means that the abstract and introduction do not include the claims made in the paper.
\item The abstract and/or introduction should clearly state the claims made, including the contributions made in the paper and important assumptions and limitations. A No or NA answer to this question will not be perceived well by the reviewers. 
\item The claims made should match theoretical and experimental results, and reflect how much the results can be expected to generalize to other settings. 
\item It is fine to include aspirational goals as motivation as long as it is clear that these goals are not attained by the paper. 
\end{itemize}

\item {\bf Limitations}
\item[] Question: Does the paper discuss the limitations of the work performed by the authors?
\item[] Answer: \answerYes{} 
\item[] Justification:  The paper discuss the limitations of the work in \autoref{sec:discussion}.
\item[] Guidelines:
\begin{itemize}
\item The answer NA means that the paper has no limitation while the answer No means that the paper has limitations, but those are not discussed in the paper. 
\item The authors are encouraged to create a separate "Limitations" section in their paper.
\item The paper should point out any strong assumptions and how robust the results are to violations of these assumptions (e.g., independence assumptions, noiseless settings, model well-specification, asymptotic approximations only holding locally). The authors should reflect on how these assumptions might be violated in practice and what the implications would be.
\item The authors should reflect on the scope of the claims made, e.g., if the approach was only tested on a few datasets or with a few runs. In general, empirical results often depend on implicit assumptions, which should be articulated.
\item The authors should reflect on the factors that influence the performance of the approach. For example, a facial recognition algorithm may perform poorly when image resolution is low or images are taken in low lighting. Or a speech-to-text system might not be used reliably to provide closed captions for online lectures because it fails to handle technical jargon.
\item The authors should discuss the computational efficiency of the proposed algorithms and how they scale with dataset size.
\item If applicable, the authors should discuss possible limitations of their approach to address problems of privacy and fairness.
\item While the authors might fear that complete honesty about limitations might be used by reviewers as grounds for rejection, a worse outcome might be that reviewers discover limitations that aren't acknowledged in the paper. The authors should use their best judgment and recognize that individual actions in favor of transparency play an important role in developing norms that preserve the integrity of the community. Reviewers will be specifically instructed to not penalize honesty concerning limitations.
\end{itemize}

\item {\bf Theory Assumptions and Proofs}
\item[] Question: For each theoretical result, does the paper provide the full set of assumptions and a complete (and correct) proof?
\item[] Answer: \answerYes{} 
\item[] Justification:  The assumptions and the proofs are provided both in the main text and the appendix.
\item[] Guidelines:
\begin{itemize}
\item The answer NA means that the paper does not include theoretical results. 
\item All the theorems, formulas, and proofs in the paper should be numbered and cross-referenced.
\item All assumptions should be clearly stated or referenced in the statement of any theorems.
\item The proofs can either appear in the main paper or the supplemental material, but if they appear in the supplemental material, the authors are encouraged to provide a short proof sketch to provide intuition. 
\item Inversely, any informal proof provided in the core of the paper should be complemented by formal proofs provided in appendix or supplemental material.
\item Theorems and Lemmas that the proof relies upon should be properly referenced. 
\end{itemize}

\item {\bf Experimental Result Reproducibility}
\item[] Question: Does the paper fully disclose all the information needed to reproduce the main experimental results of the paper to the extent that it affects the main claims and/or conclusions of the paper (regardless of whether the code and data are provided or not)?
\item[] Answer: \answerNA{} 
\item[] Justification:  This paper is a theory paper and does not have any experiments.
\item[] Guidelines:
\begin{itemize}
\item The answer NA means that the paper does not include experiments.
\item If the paper includes experiments, a No answer to this question will not be perceived well by the reviewers: Making the paper reproducible is important, regardless of whether the code and data are provided or not.
\item If the contribution is a dataset and/or model, the authors should describe the steps taken to make their results reproducible or verifiable. 
\item Depending on the contribution, reproducibility can be accomplished in various ways. For example, if the contribution is a novel architecture, describing the architecture fully might suffice, or if the contribution is a specific model and empirical evaluation, it may be necessary to either make it possible for others to replicate the model with the same dataset, or provide access to the model. In general. releasing code and data is often one good way to accomplish this, but reproducibility can also be provided via detailed instructions for how to replicate the results, access to a hosted model (e.g., in the case of a large language model), releasing of a model checkpoint, or other means that are appropriate to the research performed.
\item While NeurIPS does not require releasing code, the conference does require all submissions to provide some reasonable avenue for reproducibility, which may depend on the nature of the contribution. For example
\begin{enumerate}
\item If the contribution is primarily a new algorithm, the paper should make it clear how to reproduce that algorithm.
\item If the contribution is primarily a new model architecture, the paper should describe the architecture clearly and fully.
\item If the contribution is a new model (e.g., a large language model), then there should either be a way to access this model for reproducing the results or a way to reproduce the model (e.g., with an open-source dataset or instructions for how to construct the dataset).
\item We recognize that reproducibility may be tricky in some cases, in which case authors are welcome to describe the particular way they provide for reproducibility. In the case of closed-source models, it may be that access to the model is limited in some way (e.g., to registered users), but it should be possible for other researchers to have some path to reproducing or verifying the results.
\end{enumerate}
\end{itemize}

\item {\bf Open access to data and code}
\item[] Question: Does the paper provide open access to the data and code, with sufficient instructions to faithfully reproduce the main experimental results, as described in supplemental material?
\item[] Answer: \answerNA{} 
\item[] Justification: This paper is a theory paper and does not have any experiments.
\item[] Guidelines:
\begin{itemize}
\item The answer NA means that paper does not include experiments requiring code.
\item Please see the NeurIPS code and data submission guidelines (\url{https://nips.cc/public/guides/CodeSubmissionPolicy}) for more details.
\item While we encourage the release of code and data, we understand that this might not be possible, so “No” is an acceptable answer. Papers cannot be rejected simply for not including code, unless this is central to the contribution (e.g., for a new open-source benchmark).
\item The instructions should contain the exact command and environment needed to run to reproduce the results. See the NeurIPS code and data submission guidelines (\url{https://nips.cc/public/guides/CodeSubmissionPolicy}) for more details.
\item The authors should provide instructions on data access and preparation, including how to access the raw data, preprocessed data, intermediate data, and generated data, etc.
\item The authors should provide scripts to reproduce all experimental results for the new proposed method and baselines. If only a subset of experiments are reproducible, they should state which ones are omitted from the script and why.
\item At submission time, to preserve anonymity, the authors should release anonymized versions (if applicable).
\item Providing as much information as possible in supplemental material (appended to the paper) is recommended, but including URLs to data and code is permitted.
\end{itemize}

\item {\bf Experimental Setting/Details}
\item[] Question: Does the paper specify all the training and test details (e.g., data splits, hyperparameters, how they were chosen, type of optimizer, etc.) necessary to understand the results?
\item[] Answer: \answerNA{} 
\item[] Justification: This paper is a theory paper and does not have any experiments.
\item[] Guidelines:
\begin{itemize}
\item The answer NA means that the paper does not include experiments.
\item The experimental setting should be presented in the core of the paper to a level of detail that is necessary to appreciate the results and make sense of them.
\item The full details can be provided either with the code, in appendix, or as supplemental material.
\end{itemize}

\item {\bf Experiment Statistical Significance}
\item[] Question: Does the paper report error bars suitably and correctly defined or other appropriate information about the statistical significance of the experiments?
\item[] Answer: \answerNA{} 
\item[] Justification: This paper is a theory paper and does not have any experiments.
\item[] Guidelines:
\begin{itemize}
\item The answer NA means that the paper does not include experiments.
\item The authors should answer "Yes" if the results are accompanied by error bars, confidence intervals, or statistical significance tests, at least for the experiments that support the main claims of the paper.
\item The factors of variability that the error bars are capturing should be clearly stated (for example, train/test split, initialization, random drawing of some parameter, or overall run with given experimental conditions).
\item The method for calculating the error bars should be explained (closed form formula, call to a library function, bootstrap, etc.)
\item The assumptions made should be given (e.g., Normally distributed errors).
\item It should be clear whether the error bar is the standard deviation or the standard error of the mean.
\item It is OK to report 1-sigma error bars, but one should state it. The authors should preferably report a 2-sigma error bar than state that they have a 96\% CI, if the hypothesis of Normality of errors is not verified.
\item For asymmetric distributions, the authors should be careful not to show in tables or figures symmetric error bars that would yield results that are out of range (e.g. negative error rates).
\item If error bars are reported in tables or plots, The authors should explain in the text how they were calculated and reference the corresponding figures or tables in the text.
\end{itemize}

\item {\bf Experiments Compute Resources}
\item[] Question: For each experiment, does the paper provide sufficient information on the computer resources (type of compute workers, memory, time of execution) needed to reproduce the experiments?
\item[] Answer: \answerNA{} 
\item[] Justification: This paper is a theory paper and does not have any experiments.
\item[] Guidelines:
\begin{itemize}
\item The answer NA means that the paper does not include experiments.
\item The paper should indicate the type of compute workers CPU or GPU, internal cluster, or cloud provider, including relevant memory and storage.
\item The paper should provide the amount of compute required for each of the individual experimental runs as well as estimate the total compute. 
\item The paper should disclose whether the full research project required more compute than the experiments reported in the paper (e.g., preliminary or failed experiments that didn't make it into the paper). 
\end{itemize}

\item {\bf Code Of Ethics}
\item[] Question: Does the research conducted in the paper conform, in every respect, with the NeurIPS Code of Ethics \url{https://neurips.cc/public/EthicsGuidelines}?
\item[] Answer: \answerYes{} 
\item[] Justification: The research adheres to the NeurIPS Code of Ethics.

\item[] Guidelines:
\begin{itemize}
\item The answer NA means that the authors have not reviewed the NeurIPS Code of Ethics.
\item If the authors answer No, they should explain the special circumstances that require a deviation from the Code of Ethics.
\item The authors should make sure to preserve anonymity (e.g., if there is a special consideration due to laws or regulations in their jurisdiction).
\end{itemize}

\item {\bf Broader Impacts}
\item[] Question: Does the paper discuss both potential positive societal impacts and negative societal impacts of the work performed?
\item[] Answer: \answerNA{} 
\item[] Justification: This paper is a theory paper, and we do not forsee any direct societal implications arising from the research.
\item[] Guidelines: 
\begin{itemize}
\item The answer NA means that there is no societal impact of the work performed.
\item If the authors answer NA or No, they should explain why their work has no societal impact or why the paper does not address societal impact.
\item Examples of negative societal impacts include potential malicious or unintended uses (e.g., disinformation, generating fake profiles, surveillance), fairness considerations (e.g., deployment of technologies that could make decisions that unfairly impact specific groups), privacy considerations, and security considerations.
\item The conference expects that many papers will be foundational research and not tied to particular applications, let alone deployments. However, if there is a direct path to any negative applications, the authors should point it out. For example, it is legitimate to point out that an improvement in the quality of generative models could be used to generate deepfakes for disinformation. On the other hand, it is not needed to point out that a generic algorithm for optimizing neural networks could enable people to train models that generate Deepfakes faster.
\item The authors should consider possible harms that could arise when the technology is being used as intended and functioning correctly, harms that could arise when the technology is being used as intended but gives incorrect results, and harms following from (intentional or unintentional) misuse of the technology.
\item If there are negative societal impacts, the authors could also discuss possible mitigation strategies (e.g., gated release of models, providing defenses in addition to attacks, mechanisms for monitoring misuse, mechanisms to monitor how a system learns from feedback over time, improving the efficiency and accessibility of ML).
\end{itemize}

\item {\bf Safeguards}
\item[] Question: Does the paper describe safeguards that have been put in place for responsible release of data or models that have a high risk for misuse (e.g., pretrained language models, image generators, or scraped datasets)?
\item[] Answer: \answerNA{} 
\item[] Justification: This paper is a theory paper and does not have any experiments.
\item[] Guidelines:
\begin{itemize}
\item The answer NA means that the paper poses no such risks.
\item Released models that have a high risk for misuse or dual-use should be released with necessary safeguards to allow for controlled use of the model, for example by requiring that users adhere to usage guidelines or restrictions to access the model or implementing safety filters. 
\item Datasets that have been scraped from the Internet could pose safety risks. The authors should describe how they avoided releasing unsafe images.
\item We recognize that providing effective safeguards is challenging, and many papers do not require this, but we encourage authors to take this into account and make a best faith effort.
\end{itemize}

\item {\bf Licenses for existing assets}
\item[] Question: Are the creators or original owners of assets (e.g., code, data, models), used in the paper, properly credited and are the license and terms of use explicitly mentioned and properly respected?
\item[] Answer: \answerYes{} 
\item[] Justification: The paper accurately credits the existing results by citing the papers that inspired the techniques used in this work. 
\item[] Guidelines:
\begin{itemize}
\item The answer NA means that the paper does not use existing assets.
\item The authors should cite the original paper that produced the code package or dataset.
\item The authors should state which version of the asset is used and, if possible, include a URL.
\item The name of the license (e.g., CC-BY 4.0) should be included for each asset.
\item For scraped data from a particular source (e.g., website), the copyright and terms of service of that source should be provided.
\item If assets are released, the license, copyright information, and terms of use in the package should be provided. For popular datasets, \url{paperswithcode.com/datasets} has curated licenses for some datasets. Their licensing guide can help determine the license of a dataset.
\item For existing datasets that are re-packaged, both the original license and the license of the derived asset (if it has changed) should be provided.
\item If this information is not available online, the authors are encouraged to reach out to the asset's creators.
\end{itemize}

\item {\bf New Assets}
\item[] Question: Are new assets introduced in the paper well documented and is the documentation provided alongside the assets?
\item[] Answer: \answerNA{} 
\item[] Justification: The paper does not introduce any new assets.
\item[] Guidelines:
\begin{itemize}
\item The answer NA means that the paper does not release new assets.
\item Researchers should communicate the details of the dataset/code/model as part of their submissions via structured templates. This includes details about training, license, limitations, etc. 
\item The paper should discuss whether and how consent was obtained from people whose asset is used.
\item At submission time, remember to anonymize your assets (if applicable). You can either create an anonymized URL or include an anonymized zip file.
\end{itemize}

\item {\bf Crowdsourcing and Research with Human Subjects}
\item[] Question: For crowdsourcing experiments and research with human subjects, does the paper include the full text of instructions given to participants and screenshots, if applicable, as well as details about compensation (if any)? 
\item[] Answer: \answerNA{} 
\item[] Justification: The paper does not involve crowdsourcing nor research with human subjects.
\item[] Guidelines:
\begin{itemize}
\item The answer NA means that the paper does not involve crowdsourcing nor research with human subjects.
\item Including this information in the supplemental material is fine, but if the main contribution of the paper involves human subjects, then as much detail as possible should be included in the main paper. 
\item According to the NeurIPS Code of Ethics, workers involved in data collection, curation, or other labor should be paid at least the minimum wage in the country of the data collector. 
\end{itemize}

\item {\bf Institutional Review Board (IRB) Approvals or Equivalent for Research with Human Subjects}
\item[] Question: Does the paper describe potential risks incurred by study participants, whether such risks were disclosed to the subjects, and whether Institutional Review Board (IRB) approvals (or an equivalent approval/review based on the requirements of your country or institution) were obtained?
\item[] Answer: \answerNA{} 
\item[] Justification: The paper does not involve crowdsourcing nor research with human subjects.
\item[] Guidelines:
\begin{itemize}
\item The answer NA means that the paper does not involve crowdsourcing nor research with human subjects.
\item Depending on the country in which research is conducted, IRB approval (or equivalent) may be required for any human subjects research. If you obtained IRB approval, you should clearly state this in the paper. 
\item We recognize that the procedures for this may vary significantly between institutions and locations, and we expect authors to adhere to the NeurIPS Code of Ethics and the guidelines for their institution. 
\item For initial submissions, do not include any information that would break anonymity (if applicable), such as the institution conducting the review.
\end{itemize}

\end{enumerate}

\end{document}